%% file: main.tex
\documentclass[conference]{IEEEtran}
\usepackage{times}

\input{header}
\input{macros}

\begin{document}
\title{Entropy Regularized Motion Planning \\via Stein Variational Inference}

\author{
        \authorblockN{Alexander Lambert\authorrefmark{1} and Byron Boots\authorrefmark{2}}
        \authorblockA{ Georgia Institute of Technology\authorrefmark{1}, University of Washington\authorrefmark{2}\\
        \texttt{alambert6@gatech.edu}, \texttt{bboots@cs.washington.edu}}
}

\maketitle

\begin{abstract}
Many Imitation and Reinforcement Learning approaches rely on the availability of expert-generated demonstrations for learning policies or value functions from data. Obtaining a reliable distribution of trajectories from motion planners is non-trivial, since it must broadly cover the space of states likely to be encountered during execution while also satisfying task-based constraints. We propose a sampling strategy based on variational inference to generate distributions of feasible, low-cost trajectories for high-dof motion planning tasks. This includes a distributed, particle-based motion planning algorithm which leverages a structured graphical representations for inference over multi-modal posterior distributions. We also make explicit connections to both approximate inference for trajectory optimization and entropy-regularized reinforcement learning. Video available here: \textcolor{urldarkblue}{\url{https://youtu.be/A15MbLhRAb4}}.
\end{abstract}

\IEEEpeerreviewmaketitle

\section{Motion Planning as Probabilistic Inference}
\blfootnote{\kern-1em RSS 2021 Workshop on Integrating Planning and Learning}
For a system with state $\x \in \real^d$ and dimesion $d$, we can define a trajectory as the continuous-time function $\traj \triangleq \x(t):t \rightarrow \real^d$. Given a start state $\x_0$, trajectory optimization aims to find the optimal trajectory $\traj^*$ which minimizes an objective functional $\mathcal{F}(\traj; \x_0)$. The latter encodes a penalty on non-smooth trajectories, and might include a cost on distance to a desired goal-state $\x_g$ for minimizing trajectory length. Since the solution must be \textit{feasible} and avoid collisions with obstacles, this requirement can be imposed by either inequality constraints, $\mathcal{H}(\tau) \leq 0$~\cite{trajopt}, or by including an additional penalty in the objective~\cite{ratliff2009chomp,zucker2013chomp}:
\begin{align}
    \mathcal{F}(\traj; \x_0) =  \frac{1}{\lambda}\mathcal{F}_{obs}(\traj) + \mathcal{F}_{smooth}(\traj; \x_0)
\end{align}
for a scalar regularization weight $\frac{1}{\lambda}$, where $\lambda > 0$. To frame the optimization as an inference problem, we introduce an auxiliary binary random variable $\O\in\{0,1\}$ to indicate optimality, similarly to \cite{rawlik2010approximate, levine_tutorial}. We can then express the posterior over low-cost trajectories as: $p(\traj\,|\,\O=1) \propto p(\O=1\,|\,\traj)\  p(\traj)$, for an optimality-likelihood $p(\O=1\,|\,\traj)$ and prior probability $p(\traj)$. The optimal trajectory can then be derived by \textit{maximum a posteriori} inference \eg\ minimizing the negative log of the posterior distribution and returning its mode. For convenience, we assume the distributions belong to the exponential family (see Appendix~\ref{app:map_plan} for full derivation):
\begin{align}\label{eq:map_infer}
    \traj^* &= \argmin_{\traj} - \log p(\traj\,|\,\O=1)\\
    &= \argmin_{\traj} \frac{1}{\lambda} \mathcal{F}_{obs}(\traj) + \mathcal{F}_{smooth}(\traj; \x_0)
\end{align}

\begin{figure}[t]
    \centering
    \includegraphics[width=0.8\linewidth]{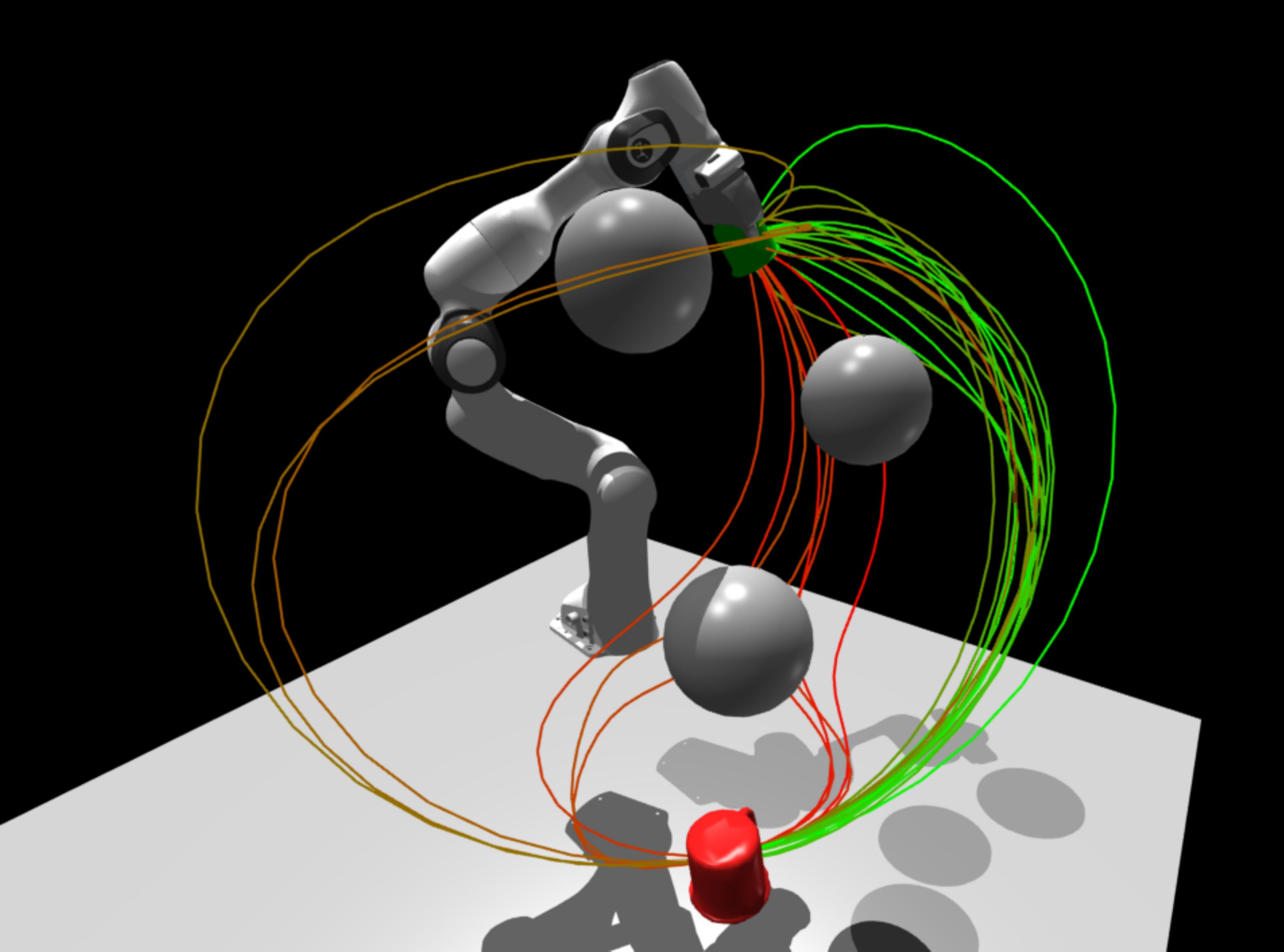}
    \captionsetup{size=footnotesize}
    	\caption{\textbf{7-Dof Trajectory Optimization.} Example of a converged trajectory distribution resulting from the proposed inference procedure. Each trajectory is a sequence of kinematic states in configuration space (joint position and velocity), and is represented by the trace of cartesian end-effector positions. Each trajectory reaches the goal position (red cup) and is collision-free with respect to three spherical obstacles. Lower-cost trajectories are marked in green and higher-cost in red. Even in a single homotopy class, the posterior distribution can be multi-modal when conditioned on cost.}\vspace{-0.6cm}
    	\label{fig:obst_distributions}
\end{figure}

\subsection{Gaussian Process Motion Planning}
\label{sec:gpmp}

The connection between approximate inference and motion planning is particularly exemplified by the GPMP class of planning algorithms~\cite{mukadam2016gaussian,mukadam2018continuous}. In this case, smoothness on generated trajectories is introduced by using a Gaussian Process prior: $p_{gp}(\tau) = \mathcal{GP}(\mathbf{\mu}(t),\, \mathcal{K}(t,t'))$, with mean function $\mu$ and covariance function $\mathcal{K}$. The distribution is then parametrized by a discrete set of support states $\thetab \triangleq [\thetab_0, ..., \thetab_N]^\top$ with a prior $p(\thetab) = \mathcal{N}(\mu, \mathcal{K})$ having a mean $\mu =[\mu(t_0), ..., \mu(t_N)]^\top$ and covariance matrix $\mathcal{K} = [\mathcal{K}(t_n,t_m)]\big|_{nm, 0\leq n,m \leq N}$. Instead of optimizing over continuous trajectory-space, we can aim to first find an optimal sparse representation, $\thetab^*$, then generate dense trajectories $\tau$ by GP-interpolation for refined collision checking and system execution~\cite{barfoot2014batch, mukadam2018continuous}. We can formulate the inference problem by augmenting the posterior distribution to include the parameter variable $\thetab$:
\begin{align}
    p(\traj, \thetab\,|\,\O=1)\ &\propto\ p(\O=1\,|\,\traj, \thetab)\, p_{gp}(\traj\,|\,\thetab)\, p(\thetab) 
\end{align}
We obtain a posterior over parameters $\thetab$ by marginalizing over trajectories $\tau$:
\begin{align}
    &p(\thetab\,|\,\O=1) \propto \int p(\O=1\,|\,\traj, \thetab)\,p_{gp}(\traj\,|\,\thetab)\, d\tau \  p(\thetab) \label{eq:posterior_1} \\
    &\ \propto \expect{p_{gp}}{ \exp\big(-\frac{1}{\lambda}C(\tau)\big)\,\Big|\, \thetab}\, \exp\big(-\frac{1}{\lambda}C(\thetab)\big) \, p(\thetab) \label{eq:posterior_2} \\
    &\ \propto  p(\O=1\,|\,\thetab)\ p(\thetab)
\end{align}
where we use the obstacle cost notation $C(\tau) = \mathcal{F}_{obs}(\tau)$ for brevity. The marginal likelihood (integral term) in \cref{eq:posterior_1} combines support state and interpolated GP collision costs. The likelihood is partitioned in \cref{eq:posterior_2} as collision costs are applied separately on $\tau$ and $\thetab$. GPMP2~\cite{Mukadam-RSS-17} proceeds by invoking a factored form of the distribution (see Appendix~\ref{app:gpmp_terms} for further discussion). We will do the same, but leave the notation in the present form for further discussion.

\section{Motion Planning as KL-Minimization}

\subsection{Bayesian posterior estimation}
Instead of only finding the mode, we might want to obtain the full posterior distribution over trajectories. This is desireable, for instance, in the context of imitation learning, where generating target distributions of expert demonstrations is necessary for learning a policy in a sample efficient manner~\cite{ke2021imitation}. One may also want to consider a distribution of expert trajectories for guiding reinforcement learning algorithms~\cite{jeong2020learning} and generating target value function estimates for energy-based policies~\cite{springenberg2020local}. This can be done by minimizing the (reverse) Kullback-Liebler divergence between a proposal distribution $q$ and the posterior:
\begin{align}\label{eq:var_objective}
\min_{q\in \mathcal{Q}}\ \lambda \dkldiv{q(\tau)}{p(\tau\,|\,\O=1)}
\end{align}
where $\min \dkldiv{q}{p} = 0$ if the proposal distribution matches the target posterior: $q(\tau) = p(\tau\,|\,\O=1)$. By selecting an optimality likelihood to be the exponentiated negative cost, $p(\O=1\,|\,\traj) \propto \exp(-\frac{1}{\lambda}C(\traj))$, we can recover the KL-regularized objective 
(see Appendix \ref{app:proof_kl_tau}).
\begin{align}\label{eq:kl_reg}
    \min_{q\in \mathcal{Q}} \mathcal{J}(q; \x_0) &=  \min_{q\in \mathcal{Q}}\  \expect{q}{C(\traj)} + \lambda\, \dkldiv{q(\traj)}{p(\traj)}
\end{align}
 which bears close resemblance to the entropy-regularized reinforcement learning literature~\cite{springenberg2020local,haarnoja2017reinforcement}.
For a trajectory distribution parametrized by a set of discrete support states $\thetab$, we can write an alternative objective:
\begin{align}\label{eq:kl_reg_theta}
q^* &= \argmin_{q\in \mathcal{Q}}\ \lambda \dkldiv{q(\thetab)}{p(\thetab\,|\,\O=1)}
\end{align}
where the solution also minimizes an entropy-regularized objective similar to \cref{eq:kl_reg} (see Appendix~\ref{app:proof_kl_theta}).

\subsection{Value approximation for Entropy Regularized Planning}
Provided that the target distribution is contained in $\mathcal{Q}$, we infer the form of the optimal $q$ from \cref{eq:var_objective} to be :
\begin{align}
    q^*(\tau) &= p(\traj\,|\,\O=1) \\
    &= \frac{p(\O=1\,|\,\traj)\ p(\traj)}{Z} \\
    &= \frac{\exp(-\frac{1}{\lambda}C(\traj))\ p(\traj)}{\int \exp(-\frac{1}{\lambda}C(\traj))\ p(\traj)\ d\traj}
\end{align}
Substituting this distribution back into the objective in \cref{eq:kl_reg}, we obtain the optimal value function, $V^*(\x_0)=\min_{q} \mathcal{J}(q; \x_0)$, to be: 
\begin{align}\label{eq:optimal_value}
V^*(\x_0) = -\temp \log \int \exp \left( -\frac{1}{\temp}C(\tau) \right) p(\traj; \x_0)\, d\tau 
\end{align}
which is otherwise known as the free-energy~\cite{theodorou2012duality}, and is the continuous-time analog to the soft-value function in the RL literature~\cite{haarnoja2017reinforcement}.
Note that this is also equal to the negative log-partition function: $V^*(\x_0) = -\lambda \log Z$, where $Z = \log \expect{p(\traj)}{p(\O=1\,|\,\traj)}$. Using the proof in Appendix~\ref{app:proof_kl_tau}, we can then directly relate the KL-divergence to the difference in values:
\begin{align}
    \dkldiv{q(\tau)}{q^*(\tau)} = \frac{1}{\lambda} \Big( V^q(\x_0) - V^*(\x_0)\Big)
\end{align}
where we use the equivalent notation for the current value estimate: $V^q(\x_0) = \mathcal{J}(q; \x_0)$. Therefore, any variational inference algorithm which is guaranteed to decrease the KL divergence at a rate $\Gamma$ will approach the target Value function at a rate $\lambda \Gamma$ from above.

Given that a similar entropy-regularized objective can be derived for support states $\thetab$ (Appendix~\ref{app:proof_kl_theta}), we can make the following equivalence:
\begin{align}
    q^*(\thetab) &= \frac{\exp(-\frac{1}{\lambda}\tilde{C}(\thetab))\ p(\thetab)}{\int \exp(-\frac{1}{\lambda}\tilde{C}(\thetab))\ p(\thetab)\ d\thetab} \\
    V^*(\x_0) &= -\temp \log \int \exp \left( -\frac{1}{\temp}\tilde{C}(\thetab) \right) p(\thetab)\, d\thetab \\
    \tilde{C}(\thetab) &= \expect{p_{gp}}{C(\tau)\,\Big|\, \thetab} + C(\thetab) \label{eq:cost_tilde}
\end{align}
where we combine the collision costs on support and interpolated GP states in \cref{eq:cost_tilde}.

\section{Stein Variational Gradient Descent}

Variational inference (VI) is a powerful tool for approximating challenging probability densities in Bayesian statistics.  As opposed to MCMC methods, VI formulates inference as an optimization problem. A proposal distribution $q(\thetab)$, belonging to a family $\mathcal{Q}$, is chosen to minimize the KL-divergence with the target posterior distribution $p(\thetab\,|x)$ over latent variable $\thetab$: 
\begin{align}
 q^*(\thetab) &= \min_{q\in\mathcal{Q}} \dkldiv{q(\thetab)}{p(\thetab\,|\,x)}
\end{align}

Traditional VI methods typically require careful selection of the distribution class $\mathcal{Q}$, which is often chosen to have a tractable parameteric form at the expense of introducing bias. 
A recently developed method,\textit{ Stein Variational Gradient Descent}~\cite{svgd_2016,matrix_svgd_2019}, avoids the challenge of determining an appropriate $\mathcal{Q}$ by leveraging a non-parameteric, particle based representation of the posterior distribution.  This approach approximates a posterior $p(\thetab | x)$  with a set of particles $\{\thetab^i\}_{i=1}^{N_p}$, $\thetab^i \in \real^p$. The particles are iteratively updated according to $\thetab^i \leftarrow \thetab^i + \epsilon \bm{\phi}^*(\thetab^i)$,  given a step-size $\epsilon$. 
The function $\bm{\phi}^*(\cdot)$ lies in the unit-ball of a reproducing kernel Hilbert space (RKHS). This RKHS is characterized by a positive-definite kernel $k(\cdot, \cdot)$. The term $\bm{\phi}^*(\cdot)$ represents the optimal perturbation or velocity field (i.e. gradient direction) which maximally decreases the KL-divergence:
\begin{align}
\bm{\phi}^*  = \argmax_{\bm{\phi} \in \mathcal{H}} \Big\{ - \nabla_\epsilon \dkldiv{ q_{\left[\epsilon \bm{\phi}\right]}}{p(\thetab\,|\,x)}\,\mathrm{s.t.}\, \norm{\bm{\phi}}_{\mathcal{H}} \leq 1\Big\},
\end{align}
where $q_{\left[\epsilon \bm{\phi} \right]}$  indicates the particle distribution resulting from taking an update step. This has been shown to yield a closed-form solution~\cite{svgd_2016}
which can be interpreted as a functional gradient in RKHS, and can be approximated with the set of particles:
\begin{align}
\hat{\bm{\phi}}^*(\thetab) = \frac{1}{N_p}\sum_{j=1}^{N_p}
\Big[
k(\thetab^j, \thetab)\nabla_{\thetab^j}\log p(\thetab^j || x)
+ \nabla_{\thetab^j} k(\thetab^j, \thetab) \Big].
\label{eq:phi_hat}
\end{align}
\cref{eq:phi_hat} has two terms that control different aspects of the algorithm. The first term is essentially a scaled gradient of the log-likelihood over the posterior's particle approximation. The second term is known as the {\em repulsive force}. Intuitively, it pushes particles apart when they get too close to each other and prevents them from collapsing into a single mode. This allows the method to approximate complex, possibly multi-modal posteriors. For the case of a  single particle, the method reduces to a standard optimization of the log-likelihood or a MAP estimate of the posterior as the repulsive force term vanishes, \textit{i.e.} $\nabla_{\thetab} k(\thetab, \thetab) = 0$. SVGD's optimization structure empirically provides better particle efficiency than other popular sampling procedures, such as Markov Chain Monte Carlo~\cite{Chen2019SteinPM}. The deterministic, gradient-based updates result in smooth transformations of the proposal distribution, a property which makes SVGD particularly attractive for trajectory optimization and inference.

\subsection{Hessian-Scaled Kernels and Second-Order SVGD}
\label{sec:hess_svgd}

As discussed in \cite{svn_2018, matrix_svgd_2019}, the convergence and accuracy of the SVGD algorithm can be largely improved by incorporating curvature information into both the kernel and the update rule in \cref{eq:phi_hat}. For instance, a positive-definite matrix $M$ can be used as a metric to scale inter-particle distances inside of an \textit{anisotropic} RBF kernel: $k(\thetab^j, \thetab^i) = \exp \big(-\frac{1}{2h}(\thetab^j - \thetab^i)^\top M (\thetab^j - \thetab^i) \big)$, where $h$ is the bandwidth parameter. Curvature information can then be shared across particles by averaging their local Hessian evaluations. Specifically, denoting the negative Hessian matrix to be $H(\thetab) = -\nabla^2_\thetab \log p(\thetab\,|\,x)$, we can define the metric $M = \frac{1}{N_p}\sum_{j=1}^{N_p} H(\thetab^j)$, which is computed using $\thetab^j$-values from the previous iteration. 
Given $M$ is constant across the particle set, it can then be used as a pre-conditioner in the SVGD update:
\begin{align}\label{eq:newton_svgd}
    \thetab^i \leftarrow \thetab^i + \epsilon M^{-1} \bm{\phi}^*(\thetab^i)
\end{align}

\section{Stein Variational Motion Planning}
\label{sec:svmp}

The inference procedure in \cref{sec:hess_svgd} can be employed for motion planning by maintaining a set of
trajectory particles, evaluating the first- and second-order gradients for planning cost and prior on each particle, and performing the update step. This would effectively result in a Newton-style version of the \textit{SV-TrajOpt} algorithm in \cite{svmpc_2020}. However, computing the full Hessian via back-propagation, for example, is prohibitively expensive for long trajectory horizons and high-dimensional state spaces. 
Instead, we can opt to leverage a sparse Gauss-Newton approximation of the Hessian, which can be derived by formulating the target distribution as a \textit{factor-graph}, given the factored representation in Appendix~\ref{app:gpmp_terms}. 
We can then express the sum of the log-likelihood and log-prior as $-\frac{1}{2}\norm{\mathbf{h}(\thetab^i)}^2_{\mathbf{\Sigma}} - \frac{1}{2}\norm{\thetab^i-\mu}^2_{\mathcal{K}}$. The log-posterior gradient term is then:
\begin{align}
 \nabla_\thetab \log p(\thetab^i\,|\,\O=1) = -\mathcal{K}^{-1}(\thetab^i-\mu) - \mathbf{J}_i^\top \mathbf{\Sigma}^{-1}\mathbf{h}(\thetab^i) \ ,
\end{align}
where $\mathbf{J}_i = \parder{\mathbf{h}}{\thetab}|_{\thetab=\thetab^i}$ is the likelihood Jacobian. The Gauss-Newton Hessian approximation is then $\mathbf{H}(\thetab^i) = \mathcal{K}^{-1} + \mathbf{J}_i^\top \mathbf{\Sigma}^{-1} \mathbf{J}_i$, and is exactly sparse block-diagonal~\cite{mukadam2018continuous}. As in \cref{sec:hess_svgd}, one way to construct a positive-definite metric $\mathbf{M}$ is to take the average of particle Hessians. The resulting matrix retains the same sparsity structure as the individual Hessian matrices. The metric can then be used to define the anisotropic kernel, and allow us to incorporate curvature information in computing the SVGD functional gradient, $\bm{\phi}^*(\thetab^i)$.
We can then compute the update to each particle by solving the linear system for $\delta \thetab^i$, in batch : $\mathbf{M}\, \delta \thetab^i = \bm{\phi}^*(\thetab^i)$. The full SVGP-MP algorithm is outlined in Appendix~\ref{app:algorithm}. Note that, for the single particle case ($N_p=1$), we recover the original GPMP2 algorithm.

\subsection{Particle-based Value Function Approximation}
The target posterior can be approximated by the empirical particle distribution $\hat{q}(\thetab^i) = \sum_{i=1}^{N_p} w^i\delta(\thetab^i)$, with weights
\begin{align}
&    w^i = \frac{\exp(-\frac{1}{\lambda} \tilde{C}(\thetab^i))\ p(\thetab^i)}{\sum_{j=1}^{N_p} \exp(-\frac{1}{\lambda} \tilde{C}(\thetab^j)) \ p(\thetab^j)}    
\end{align}
We can then provide an estimate for the optimal Value function:
\begin{align}
    \hat{V}^q(\x_0) &= -\lambda \log \sum_{j=1}^{N_p} \exp \Big(-\frac{1}{\lambda} \tilde{C}(\thetab^j) + \log p(\thetab^j) \Big)
\end{align}


\section{Experiments}
\label{sec:experiments}

For all experiments, we use the anisotropic RBF kernel, and take the metric to be the average Hessian across particles. The GPU-accelerated simulator IsaacGym~\cite{isaac-gym} was used for visualization. Preliminary results for the reaching task are shown in Appendix~\ref{app:results}, where value estimates and cost statistics with and without obstacles are presented.  Qualitative results depicting a time-lapse of converging trajectory distributions are presented in the supplementary video. We do not include a GP-interpolation factor for generating the included results, however we leave this for future work.


\section{Conclusion}
\label{sec:conclusion}
Future work will include comparison to other probabilistic motion planners~\cite{mukadam2018continuous, kalakrishnan_stomp:_2011}, a receding horizon formulation (similarly to \cite{svmpc_2020,barcelos2021dual}), and application of the proposed algorithm to imitation learning and model-based reinforcement learning problems. Significant speed-ups can also be obtained by leveraging sparse-linear algebra solvers to exploit the inherent structure, and integrating GP-interpolation to reduce the number of necessary support states needed to define particle trajectories.

\bibliographystyle{plainnat}
\bibliography{references}

\include{appendix}

\end{document}

%% file: header.tex
\usepackage[square,numbers,sort&compress]{natbib}
\usepackage[bookmarks=true]{hyperref}

\usepackage{multicol}
\usepackage{times}
\usepackage{cancel}
\usepackage{color}
\usepackage{caption}
\usepackage{subcaption}
\usepackage{graphicx}
\usepackage{inconsolata}
\usepackage{xcolor}
\usepackage{float}
\usepackage{ctable}
\usepackage{boldline}
\usepackage{cancel}
\usepackage{soul}
\usepackage{wrapfig}
\usepackage{minibox}
\usepackage{multirow}
\usepackage[ruled, noline]{algorithm2e}
\SetKwFor{ForPar}{for}{do in parallel}{end forpar}
\definecolor{alg}{RGB}{46, 149, 186}

\SetCommentSty{mycommfont}
\definecolor{urldarkblue}{RGB}{1, 111, 255}

\usepackage[T1]{fontenc} 
\usepackage[latin1]{inputenc}

\usepackage[english]{babel}
\usepackage[autostyle, english = american]{csquotes}
\MakeOuterQuote{"}

\usepackage[T1]{fontenc} 
\usepackage[latin1]{inputenc}
\usepackage{babel}
\usepackage{textcomp}
\usepackage{stmaryrd}
\usepackage{upgreek}
\usepackage{bm}
\usepackage{bbm}
\usepackage{url}
\usepackage{breakurl}
\usepackage{xspace}
\usepackage{afterpage}
\usepackage{amsmath, amsthm, amssymb}
\usepackage{mathtools}
\usepackage[capitalise]{cleveref} 

%% file: macros.tex
\newcommand{\parder}[2]{\frac{\partial #1}{\partial #2}}

\newcommand{\expect}[2]{\mathbb{E}_{#1}\Big[ #2 \Big]}

\newcommand\blfootnote[1]{
  \begingroup
  \renewcommand\thefootnote{}\footnote{#1}
  \addtocounter{footnote}{-1}
  \endgroup
}
\newcommand{\eg}{\textit{e.g.}}

\newcolumntype{L}[1]{>{\raggedright\arraybackslash}p{#1}}
\newcolumntype{C}[1]{>{\centering\arraybackslash}p{#1}}
\newcolumntype{R}[1]{>{\raggedleft\arraybackslash}p{#1}}

\def\real{\mathbb{R}}

\def\O{\mathcal{O}}

\def\J\mathbf{J}
\def\traj{\tau}

\def\x{\mathbf{x}}

\def\temp{\lambda}

\def\thetab{\mathrm{\boldsymbol{\theta}}}

\DeclarePairedDelimiterX{\infdivx}[2]{(}{)}{%
  #1\;\delimsize\|\;#2%
}
\newcommand{\dkldiv}{D_{KL}\infdivx}

\newcommand{\hide}[1]{}
\newcommand{\norm}[1]{\big|\big| #1 \big|\big|}

\DeclareMathOperator*{\argmax}{arg\,max} 
\DeclareMathOperator*{\argmin}{arg\,min}

\newtheorem{theorem}{Theorem}

%% file: appendix.tex
\begin{appendices}

\section{Motion Planning as MAP Inference}
Trajectory optimization can be frame as an inference problem:
\label{app:map_plan}
\begin{align}\label{eq:map_infer}
    \traj^* &= \argmin_{\traj} - \log p(\traj\,|\,\O=1)\\
    &= \argmin_{\traj} - \log p(\O=1\,|\,\traj)\  - \log p(\traj) \\
    &= \argmin_{\traj} - \log \exp(-\frac{1}{\lambda}\mathcal{F}_{obs}(\traj)) \\
    &\quad - \log \exp(- \mathcal{F}_{smooth}(\traj; \x_0)) \\
    &= \argmin_{\traj} \frac{1}{\lambda} \mathcal{F}_{obs}(\traj) + \mathcal{F}_{smooth}(\traj; \x_0)
\end{align}

\section{Factor graph formulation}
\label{app:gpmp_terms}

We demonstrate the equivalance in notation with the factored form described in GPMP2~\cite{mukadam2018continuous}. 

\textbf{Prior:} The prior $p(\thetab)$ corresponds to the product of start, goal, and GP-prior terms:
\begin{align}
    p(\thetab)\quad &\propto \quad f_0^p(\thetab_0) f_N^p(\thetab_N)  \prod_{n=0}^{N-1} f_n^{gp} (\thetab_n, \thetab_{n+1})\\
    &=\quad \exp\big(-\frac{1}{2}\norm{\thetab - \mu}^2_{\mathcal{K}}\big)
\end{align}

\textbf{Likelihood:} The likelihood can be expressed as the product of support state and GP-interpolated collision \textit{marginal}-likelihood terms
\begin{align}
&p(\O=1\,|\,\thetab)\\
= &\int p(\O=1\,|\,\traj, \thetab)\,p_{gp}(\traj\,|\,\thetab)\, d\tau \quad \\
= & \int \exp\big(-\frac{1}{\lambda}C(\tau)\big)\,\exp\big(-\frac{1}{\lambda}C(\thetab)\big) \,p_{gp}(\traj\,|\,\thetab)\, d\tau \quad \\
 = &\ \expect{p_{gp}}{ \exp\big(-\frac{1}{\lambda}C(\tau)\big)\,\Big|\, \thetab}\, \exp\big(-\frac{1}{\lambda}C(\thetab)\big) \\
\end{align}

Note that the first term requires marginalizing over interpolated trajectories. This can be incorporated by defining an interpolation factor on adjacent support states, allowing the following factored form of the likelihood function:
\begin{align}
&p(\O=1\,|\,\thetab)\\
= &\prod_{n=0}^N f_n^{intp}(\thetab_n,\thetab_{n+1}) f_n^{obs}(\thetab_n)\\
= &\exp\big(-\frac{1}{2\lambda}\norm{\mathbf{h}(\thetab)}^2_{\mathbf{\Sigma}}\big)
\end{align}
where in the last line, we combine factor costs into a single vector-valued cost function $\mathbf{h}(\cdot)$, weighted by the sparse precision matrix $\mathbf{\Sigma}^{-1}$. This can be expressed as the combined quadratic cost over support states:
\begin{align}
    \tilde{C}(\thetab) = \frac{1}{2}\mathbf{h}(\thetab)^\top \mathbf{\Sigma}^{-1}\mathbf{h}(\thetab) .
\end{align}

\section{KL-regularized motion planning}
\subsection{For Continuous Trajectories}
\label{app:proof_kl_tau}
\begin{theorem} The optimal proposal distribution $q^*(\tau)$ which minimizes the reverse KL-divergence (or I-projection) over continuous trajectories $\tau$ is also the solution to the KL-regularized expected-cost objective.
\end{theorem}
\begin{proof}
\begin{align}
& q^*(\tau) \\
&= \argmin_{q\in \mathcal{Q}}\ \lambda\,\dkldiv{q(\traj)}{p(\traj\,|\,\O=1)} \\
&= \argmin_{q\in \mathcal{Q}}\ \lambda\,\expect{q}{\log q(\traj) - \log p(\traj\,|\,\O=1)} \\\label{eq:bayes_step1}
&= \argmin_{q\in \mathcal{Q}}\ \lambda\,\expect{q}{\log q(\traj) - \log p(\O=1\,|\,\traj)\\
&- \log p(\traj) + \log Z} \\\label{eq:bayes_step2}
&= \argmin_{q\in \mathcal{Q}}\ \lambda\,\expect{q}{\log q(\traj) - \log p(\O=1\,|\,\traj) - \log p(\traj)} \\
&= \argmin_{q\in \mathcal{Q}}\ - \lambda\,\expect{q}{\log p(\O=1\,|\,\traj)}  + \lambda\,\expect{q}{\log\frac{q(\traj)}{p(\traj)}}\\
&= \argmin_{q\in \mathcal{Q}}\ - \lambda\,\expect{q}{\log p(\O=1\,|\,\traj)} + \lambda\, \dkldiv{q(\traj)}{p(\traj)}
\end{align}
with application of Bayes' rule in \cref{eq:bayes_step1}. 
The log-partition function $ \log Z = \log \expect{p(\traj)}{p(\O=1\,|\,\traj)} $ is a constant with respect to $q$, and can be dropped from the optimization (\cref{eq:bayes_step2}). 

By selecting an optimality likelihood to be the exponentiated negative cost: $p(\O=1\,|\,\traj) \propto \exp(-\frac{1}{\lambda}C(\traj))$, we can simplify the solution:
\begin{align}
    q^*(\tau) &= \argmin_{q\in \mathcal{Q}}\ - \lambda\,\expect{q}{\log \exp(-\frac{1}{\lambda}C(\traj))} + \lambda\, \dkldiv{q(\traj)}{p(\traj)} \\
&= \argmin_{q\in \mathcal{Q}}\ \expect{q}{C(\traj)} + \lambda\, \dkldiv{q(\traj)}{p(\traj)} ,
\end{align}
recovering the KL-regularized objective.
\end{proof}

\subsection{For Distribution Parameters}
\label{app:proof_kl_theta}

\begin{theorem} The optimal proposal distribution $q^*(\thetab)$ which minimizes the reverse KL-divergence over a discrete set of trajectory parameters $\thetab$ is also the solution to a KL-regularized expected-cost objective.
\end{theorem}
\begin{proof}
\begin{align}
& q^*(\thetab) \\
&= \argmin_{q\in \mathcal{Q}}\ \lambda\,\dkldiv{q(\thetab)}{p(\thetab\,|\,\O=1)} \label{eq:proof_theta_kl_min}\\
&= \argmin_{q\in \mathcal{Q}}\ \lambda\,\expect{q}{\log q(\thetab) - \log p(\thetab\,|\,\O=1)} \\
&= \argmin_{q\in \mathcal{Q}}\ \lambda\,\expect{q}{\log q(\thetab) \\
& \quad - \log p(\O=1\,|\,\thetab) - \log p(\thetab) + \log Z} \\
&= \argmin_{q\in \mathcal{Q}}\ \lambda\,\expect{q}{\log q(\thetab) - \log p(\O=1\,|\,\thetab) - \log p(\thetab)} \\
&= \argmin_{q\in \mathcal{Q}}\ \lambda\,\expect{q}{\log q(\thetab) - \log \expect{p_{gp}}{ \exp\big(-\frac{1}{\lambda}C(\tau)\big)\,\Big|\, \thetab}\\ 
& \hspace{40pt} - \log \exp\big(-\frac{1}{\lambda}C(\thetab)\big) - \log p(\thetab)} \\
&= \argmin_{q\in \mathcal{Q}}\ \underbrace{-\lambda \expect{q}{ \log \expect{p_{gp}}{ \exp\big(-\frac{1}{\lambda}C(\tau)\big)\,\Big|\, \thetab}}}_{\leq -\lambda \expect{q, p_{gp}}{\log \exp\big(-\frac{1}{\lambda}C(\tau)\big)\,\Big|\, \thetab}}\\
&\hspace{40pt} + \expect{q}{C(\thetab)}  + \lambda\,\expect{q}{\log\frac{q(\thetab)}{p(\thetab)}}\\[10pt]
&=  \argmin_{q\in \mathcal{Q}}\ \expect{q(\thetab)}{\expect{p_{gp}}{C(\tau)\,\Big|\, \thetab}} +\\
& \hspace{40pt} \expect{q(\thetab)}{C(\thetab)}+ \lambda\, \dkldiv{q(\thetab)}{p(\thetab)} \\[10pt]
&=  \argmin_{q\in \mathcal{Q}}\ \expect{q(\thetab)}{C(\thetab) + \expect{p_{gp}}{C(\tau)\,\Big|\, \thetab}} +\\
& \hspace{40pt} \lambda\, \dkldiv{q(\thetab)}{p(\thetab)} \\[10pt]
&=  \argmin_{q\in \mathcal{Q}}\ \expect{q(\thetab)}{\tilde{C}(\thetab)} + \lambda\, \dkldiv{q(\thetab)}{p(\thetab)}
\end{align}

where $\tilde{C}(\thetab) \equiv C(\thetab) + \expect{p_{gp}}{C(\tau)\,\Big|\, \thetab}$. The KL-regularized objective for support states $\thetab$ can then be solved by minimizing the original KL-minimization objective in \cref{eq:proof_theta_kl_min}. 
\end{proof}

\newpage
\section{Algorithm}
\label{app:algorithm}
  {\SetAlgoNoLine%
  \begin{algorithm}[h]
    \DontPrintSemicolon 
    \KwIn{
        Start state $\x_0$,
        factor-graph cost function $\mathbf{h}$,
        factor-graph cost weights $\mathbf{\Sigma}$,
        prior $p(\thetab)$,
        GP-prior mean $\mu$ and covariance $\mathcal{K}$,
        kernel $k(\cdot, \cdot)$,
        metric function $M(\cdot, \cdot)$,
        step-size $\epsilon$
        }
    \vspace{10pt}
    \tcp{Initialize particles}
    Sample $\{\thetab^i\}_{i=1}^{N_p} \sim p(\thetab) $\\[5pt]
    \While{Not Converged} {
    \tcp{Batched computation}
    \ForPar{$i = 1, 2, ...,N_p$}{\vspace{2pt} 
    \tcp{Cost-factor errors, Jacobians}
    $\mathbf{h}_i = \mathbf{h}(\thetab^i)$,\, $\mathbf{J}_i = \parder{\mathbf{h}}{\thetab}|_{\thetab=\thetab^i}$\\[5pt] 
    \tcp{Log-posterior gradient}
    $g(\thetab^i) = -\mathcal{K}^{-1}(\thetab^i-\mu) - \mathbf{J}_i^\top \mathbf{\Sigma}^{-1}\mathbf{h}(\thetab^i)$\\[5pt] 
    \tcp{Hessian}
    $\mathbf{H}(\thetab^i) = \mathcal{K}^{-1} + \mathbf{J}_i^\top \mathbf{\Sigma}^{-1} \mathbf{J}_i$\\[5pt] 
    }
    \tcp{Metric}
    $\mathbf{M} = M(\{\thetab^j\}, \{\mathbf{H}(\thetab^j)\}) $ \\[5pt] 
    \tcp{Set kernel function}
    $k_{\mathbf{M}}(\cdot, \cdot) \coloneqq k(\mathbf{M}^{1/2}\ \cdot\ , \mathbf{M}^{1/2}\ \cdot\ )$\\[5pt] 
    \tcp{Batched computation}
    \ForPar{$i = 1, 2, ..., N_p$ } {
    \tcp{SVGD gradient}
    $\hat{\bm{\phi}}^*(\thetab^i) = \frac{1}{N_p}\sum_{j=1}^{N_p} k_{\mathbf{M}}(\thetab^j, \thetab^i)g(\thetab^j) + \nabla_{\thetab^j} k_{\mathbf{M}}(\thetab^j, \thetab^i)$ \\[5pt] 
    \tcp{Solve sparse linear system}
    $\delta \thetab^i = {\mathbf{M}}^{-1} \hat{\bm{\phi}}^*(\thetab^i)$  \\[5pt]
    \tcp{Update particles}
    $\thetab^i \leftarrow \thetab^i + \epsilon\, \delta \thetab^i $}
    }\vspace{5pt}
    \caption{SVGP Motion Planning}
    \label{algo:svtrajopt}
  \end{algorithm}}%

\clearpage
\section{Preliminary Results}
\label{app:results}

\begin{figure}[h]
    \centering
    \captionsetup{size=footnotesize}
    \includegraphics[width=\linewidth]{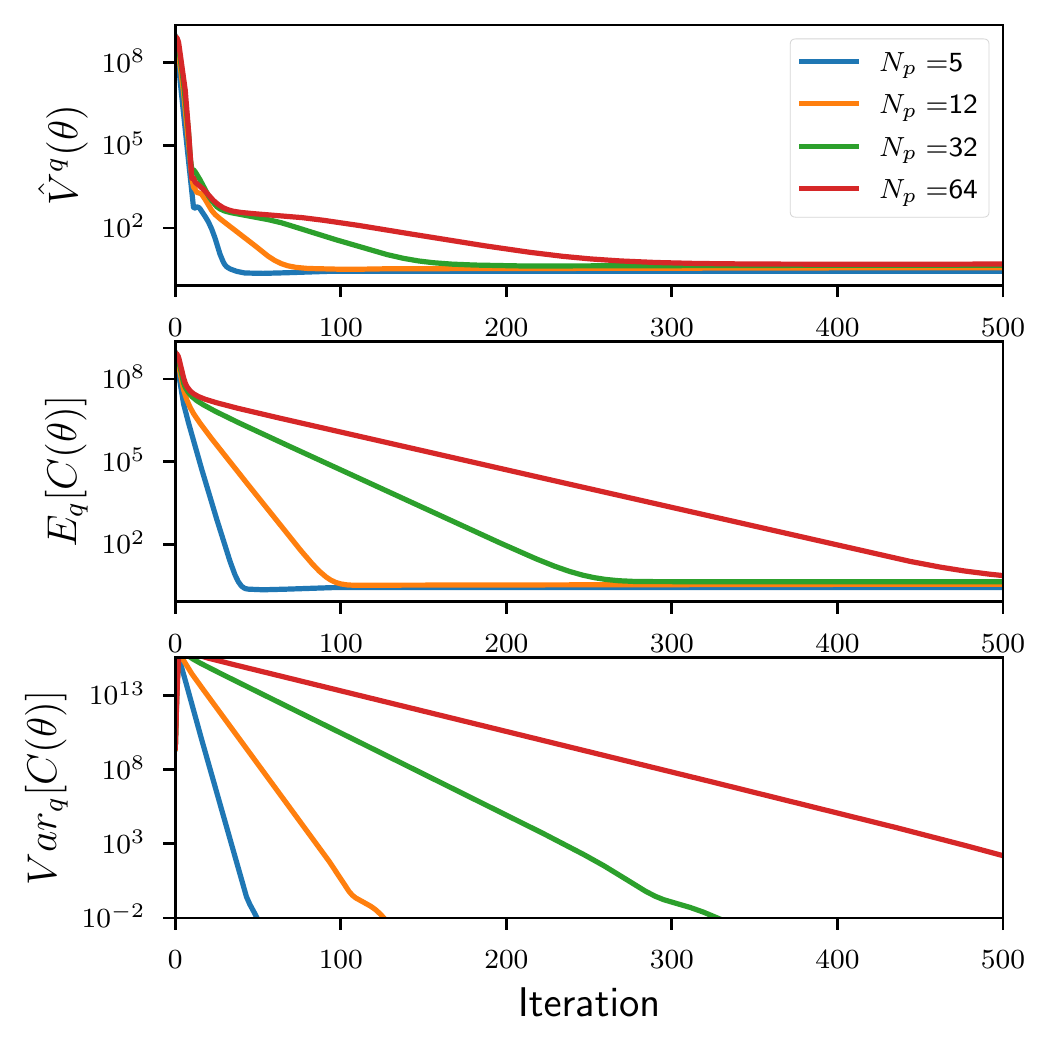}
    \caption{\textbf{Reaching in free-space.}\quad Estimates of the target value, expected cost, and cost Variance for various particle set sizes $N_p$. Results are averaged over 5 independent trials.}
    \label{fig:results_free}
\end{figure}  

Value estimates and cost statistics for the free-space reaching task are shown in \cref{fig:results_free}. The low cost-variance estimates are reflective of the high-weight placed on the goal-position cost.
Results corresponding to the three-sphere obstacle environment (\cref{fig:obst_distributions})) are found in \cref{fig:results_obst}. Particle trajectories can vary substantially, with diverse curvature, while still reaching the goal configuration (as seen in \cref{fig:obst_distributions}).

\begin{figure}[t]
    \centering
    \captionsetup{size=footnotesize}
    \includegraphics[width=\linewidth]{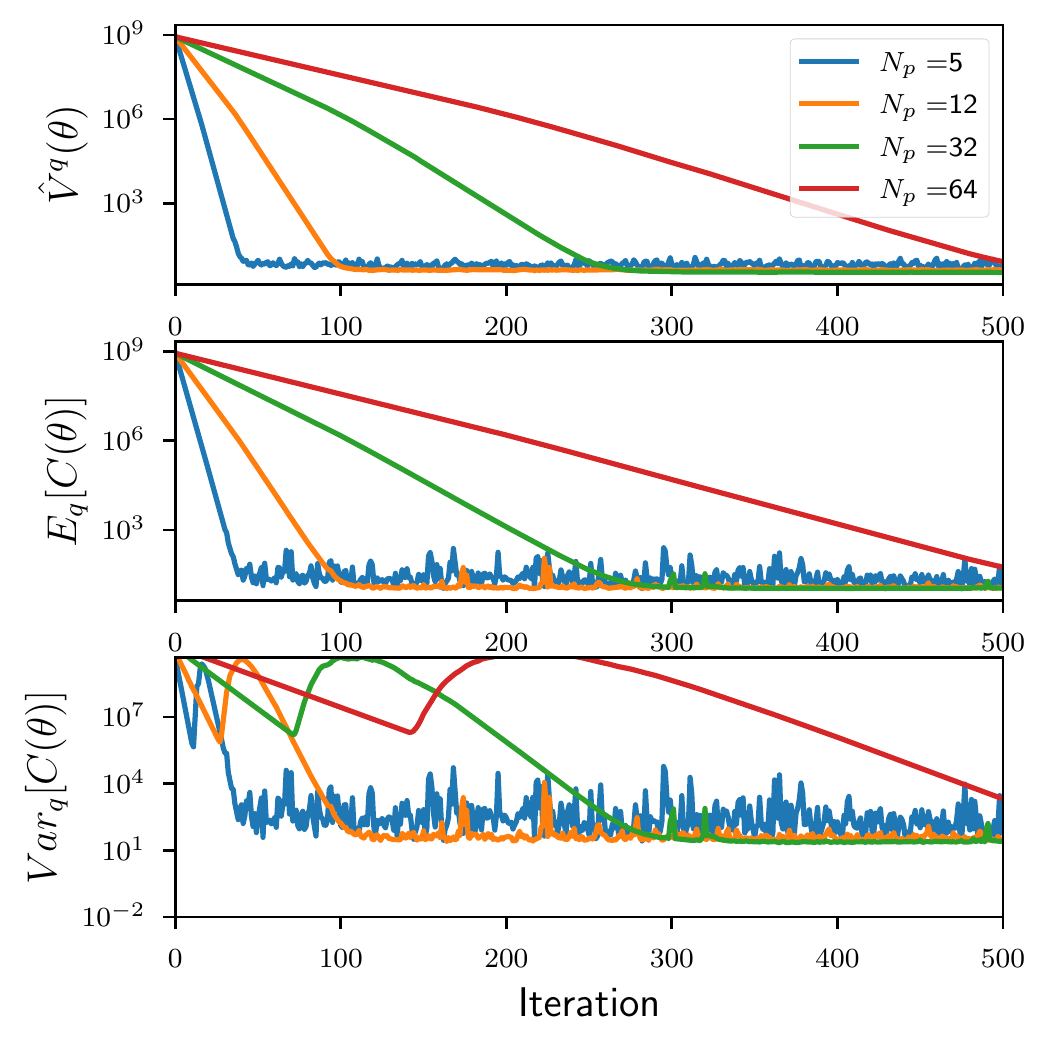}
    \caption{\textbf{Reaching around obstacles.}\quad Estimates of the target value, expected cost, and cost variance for various particle set sizes $N_p$. For comparison, a step size of $\epsilon=1$ is used in all cases. Results are averaged over 5 independent trials. Note that the logarithmic scale magnifies perturbations in lower value ranges.}
    \label{fig:results_obst}
\end{figure}  

\end{appendices}